\documentclass[11pt, a4paper]{article}

\usepackage[utf8]{inputenc}
\usepackage[T1]{fontenc}
\usepackage{geometry}
\usepackage{graphicx}
\usepackage{amsmath, amssymb, amsfonts, amsthm}
\usepackage{mathtools}
\usepackage{booktabs}
\usepackage{hyperref}
\usepackage{longtable}
\usepackage{cite}
\usepackage{bm}
\usepackage{authblk}
\usepackage{float}
\usepackage{microtype}
\usepackage[dvipsnames]{xcolor}

\hypersetup{
  colorlinks=true,
  linkcolor=MidnightBlue,
  citecolor=MidnightBlue,
  urlcolor=MidnightBlue
}

\newcommand{\R}{\mathbb{R}}

\newcommand{\M}{\mathcal{M}}
\newcommand{\TCO}{\mathcal{O}}
\newcommand{\E}{\mathbb{E}}
\newcommand{\Var}{\mathrm{Var}}
\newcommand{\cov}{\mathrm{Cov}}
\newcommand{\Tr}{\mathrm{Tr}}
\newcommand{\rank}{\mathrm{rank}}
\newcommand{\Id}{\mathrm{Id}}

\newcommand{\KL}{D_{\mathrm{KL}}}
\newcommand{\op}{\mathrm{op}}
\newcommand{\dist}{\mathrm{dist}}

\newtheorem{theorem}{Theorem}[section]

\newtheorem{corollary}[theorem]{Corollary}
\newtheorem{prop}[theorem]{Proposition}
\newtheorem{definition}{Definition}[section]
\newtheorem{hypothesis}{Hypothesis}
\newtheorem{remark}{Remark}

\title{\textbf{Latent Object Permanence: Topological Phase Transitions, Free-Energy Principles, and Renormalization Group Flows in Deep Transformer Manifolds}}

\author[1]{Faruk Alpay}
\author[2]{Bugra Kilictas}

\affil[1]{Department of Computer Engineering, Bahçeşehir University, Istanbul, Turkey\\ \texttt{faruk.alpay@bahcesehir.edu.tr}}
\affil[2]{Department of Computer Engineering, Bahçeşehir University, Istanbul, Turkey\\ \texttt{bugra.kilictas@bahcesehir.edu.tr}}

\date{\today}

\begin{document}
\maketitle

\begin{abstract}
The emergent capability of Large Language Models (LLMs) to perform multi-step reasoning suggests an internal mechanism that behaves as if it discretizes continuous representations into stable, reusable units. We model this phenomenon as a \emph{phase transition} in the geometry of a latent activation manifold $\M$ across depth. Let $h^{(l)}\in\R^d$ be the residual stream at layer $l$ and $C^{(l)}=\cov(h^{(l)})$ its covariance. We study (i) the spectral density $\rho_l(\lambda)$ of $C^{(l)}$, (ii) intrinsic-dimension proxies such as effective rank, and (iii) a sparsity-like order parameter based on \emph{Object Integrity} $\Omega$. We propose that sufficiently large models exhibit a critical depth fraction $\gamma_c\approx 0.42$ where symmetry breaking occurs: the latent dynamics shift from a high-entropy ``liquid'' regime (diffuse superposition, Marchenko--Pastur-like bulk) to a low-entropy ``solid'' regime (spectral gaps, stable basins). We formalize these basins as \textit{Transient Class Objects} (TCOs), and connect their emergence to (a) a free-energy variational principle underlying attention softmax and (b) a discrete Renormalization Group (RG) flow that contracts irrelevant directions. We provide sufficient conditions guaranteeing spectral collapse (transverse contraction) and give rigorous mixture-model results relating logical separability to low-rank spiked covariance structure.
\end{abstract}

\section{Introduction}

Interpretability work on Transformer models \cite{vaswani} often treats the latent space $\R^d$ as a continuous semantic field where meaning is encoded in approximately linear directions \cite{mikolov, park}. Yet multi-step reasoning requires operations that are effectively discrete: negation, quantification, variable binding, and compositional control flow. Bottleneck hypotheses---e.g.\ the \emph{Consciousness Prior} \cite{bengio} and capsule-like factorization \cite{hinton}---suggest that high-level cognition requires sparse, manipulable factors that behave like latent ``objects.''

We investigate whether deep Transformers spontaneously implement such discretization via a mechanism analogous to the Renormalization Group (RG) \cite{mehta}: a coarse-graining flow that integrates out short-range correlations (local syntax) and stabilizes long-range operators (logical/semantic relations). Unlike shallow-layer accounts emphasizing feature superposition \cite{elhage, olah}, we focus on deep-layer regimes where reasoning emerges, and ask whether the latent geometry exhibits signatures of a phase transition.

\paragraph{Core thesis.}
At sufficient scale, depth acts like an implicit \emph{cooling schedule}: attention becomes sharper, free energy decreases, covariance spectra develop spikes and gaps, and effective dimensionality collapses. We interpret the post-critical regime as a ``solid'' phase in which latent trajectories concentrate near stable basins (TCOs) supporting object permanence across steps.

\section{Preliminaries and Observables}

\subsection{Depth as Discrete Time and Pushforward Dynamics}

Consider a Transformer with $L$ residual blocks. Let $h^{(l)}\in\R^d$ denote the (tokenwise) residual stream at layer $l$. Define normalized depth
\begin{equation}
\gamma \coloneqq \frac{l}{L}\in[0,1].
\end{equation}
Treat each layer as a measurable map $F_l:\R^d\to\R^d$ (including attention, MLP, residual addition). Then the distribution $\mathsf{P}_l$ of $h^{(l)}$ evolves by pushforward
\begin{equation}
\mathsf{P}_{l+1} = (F_l)_{\#}\mathsf{P}_l.
\end{equation}

\subsection{Covariance Spectrum, Effective Rank, and Participation Ratio}

Let $\mu^{(l)}=\E[h^{(l)}]$ and
\begin{equation}
C^{(l)}\coloneqq \E\big[(h^{(l)}-\mu^{(l)})(h^{(l)}-\mu^{(l)})^\top\big]\succeq 0.
\end{equation}
Let eigenvalues be $\lambda^{(l)}_1\ge \cdots\ge \lambda^{(l)}_d\ge 0$, and define the spectral density (empirical measure)
\begin{equation}
\rho_l(\lambda)\coloneqq \frac{1}{d}\sum_{i=1}^d \delta(\lambda-\lambda^{(l)}_i).
\end{equation}

Define normalized eigenvalues $\hat{\lambda}^{(l)}_i=\lambda^{(l)}_i/\Tr(C^{(l)})$ and spectral entropy
\begin{equation}
S(C^{(l)})\coloneqq -\sum_{i=1}^d \hat{\lambda}^{(l)}_i\log \hat{\lambda}^{(l)}_i,
\qquad
R_{\mathrm{eff}}(C^{(l)})\coloneqq \exp(S(C^{(l)})).
\end{equation}
Also define the participation ratio (PR) dimension
\begin{equation}
d_{\mathrm{PR}}(C^{(l)})\coloneqq \frac{\Tr(C^{(l)})^2}{\Tr\!\big((C^{(l)})^2\big)}.
\end{equation}

\begin{prop}[Basic bounds]
For any $C\succeq 0$ with $C\neq 0$,
\[
1 \le d_{\mathrm{PR}}(C)\le \rank(C)\le d,
\qquad
1 \le R_{\mathrm{eff}}(C)\le d.
\]
\end{prop}

\subsection{Object Integrity Order Parameter}

For $h\in\R^d\setminus\{0\}$ define
\begin{equation}
\Omega(h)\coloneqq 1-\frac{\|h\|_1}{\sqrt{d}\,\|h\|_2}.
\end{equation}

\begin{prop}[Sharp bounds and extremizers]
\label{prop:omega_bounds}
For any $h\neq 0$,
\[
0\le \Omega(h)\le 1-\frac{1}{\sqrt{d}}.
\]
Moreover, $\Omega(h)=0$ iff $|h_1|=\cdots=|h_d|$, and $\Omega(h)=1-\frac{1}{\sqrt{d}}$ iff $h$ is 1-sparse.
\end{prop}

\begin{proof}
By Cauchy--Schwarz, $\|h\|_1\le \sqrt{d}\|h\|_2$ gives $\Omega(h)\ge 0$ with equality iff all $|h_i|$ equal.
Also $\|h\|_1\ge \|h\|_2$ gives $\Omega(h)\le 1-1/\sqrt{d}$ with equality iff $h$ is 1-sparse.
\end{proof}

Define the depth profile mean and susceptibility
\begin{equation}
m(\gamma)\coloneqq \E[\Omega(h^{(l)})],\qquad \chi(\gamma)\coloneqq \Var(\Omega(h^{(l)})),
\quad \gamma=l/L.
\end{equation}

\section{Information Geometry of the Latent Manifold}

\subsection{Fisher Metric Induced by the Output Distribution}

Let $P_\theta(y\mid h)$ denote the next-token distribution defined by the output head (e.g.\ linear map + softmax). The Fisher information metric on latent coordinates is
\begin{equation}
g_{ij}(h)\coloneqq \E_{y\sim P_\theta(\cdot\mid h)}
\left[\frac{\partial \log P_\theta(y\mid h)}{\partial h_i}\frac{\partial \log P_\theta(y\mid h)}{\partial h_j}\right].
\end{equation}
This metric quantifies local sensitivity of predictions to perturbations in $h$, and thus is a natural candidate for a task-relevant Riemannian structure on $\M$.

\begin{remark}
If the output head is linear logits $\ell=W_{\text{out}}h+b$, then $\partial_h \log P(y\mid h)=W_{\text{out}}^\top (e_y - p)$, so $g(h)$ is controlled by $W_{\text{out}}$ and the output covariance $\cov(e_y)$ under $p$. This makes $g$ empirically estimable (up to sampling) from logits and $W_{\text{out}}$.
\end{remark}

\subsection{Curvature as a Proxy for Hierarchy}

Curvature quantities (e.g.\ Ricci curvature) measure how geodesics converge/diverge and can encode representational hierarchy. While we do not assume constant curvature, we propose the following as a \emph{diagnostic hypothesis}:

\begin{definition}[Hyperbolic embedding hypothesis (diagnostic)]
Deep semantic hierarchies are facilitated when the effective latent geometry exhibits negative curvature in relevant subspaces. Early layers are expected to behave closer to locally Euclidean geometry (syntax), while deeper layers may induce more negatively curved effective geometry (hierarchical semantics).
\end{definition}

A convenient conceptual model is a forced geodesic equation on $(\M,g)$ driven by an attention-induced potential $U_{\mathrm{attn}}$:
\begin{equation}
\frac{D^2 h^i}{d\tau^2} + \Gamma^{i}_{jk}(h)\frac{dh^j}{d\tau}\frac{dh^k}{d\tau}
= -g^{ij}(h)\frac{\partial U_{\mathrm{attn}}}{\partial h^j},
\end{equation}
where $D/d\tau$ is the covariant derivative and $\Gamma^i_{jk}$ are Christoffel symbols. The forced term is a schematic way to express that attention reshapes trajectories towards preferred basins.

\section{Thermodynamics of Attention: A Free-Energy Principle}

\subsection{Softmax as a Gibbs Distribution}

Consider one attention head with query $q\in\R^{d_k}$, keys $\{k_j\}_{j=1}^T$, values $\{v_j\}_{j=1}^T$. The attention weights are
\begin{equation}
a_j = \frac{\exp(\beta \langle q,k_j\rangle)}{\sum_{r=1}^T \exp(\beta \langle q,k_r\rangle)},
\qquad \beta \coloneqq \frac{1}{\sqrt{d_k}}.
\end{equation}
Define energies $E_j\coloneqq -\langle q,k_j\rangle$. Then $a_j\propto \exp(-\beta E_j)$ is a Gibbs distribution.

\begin{prop}[Variational characterization of attention via free energy]
\label{prop:free_energy_attention}
Let $\Delta_T=\{p\in\R^T_{\ge 0}:\sum_j p_j=1\}$. Define the free-energy functional
\begin{equation}
\mathcal{F}(p)\coloneqq \sum_{j=1}^T p_j E_j + \frac{1}{\beta}\sum_{j=1}^T p_j \log p_j.
\end{equation}
Then $\mathcal{F}(p)$ is minimized over $\Delta_T$ uniquely by the Gibbs distribution
\[
p_j^\star = \frac{\exp(-\beta E_j)}{\sum_{r=1}^T \exp(-\beta E_r)}=\frac{\exp(\beta \langle q,k_j\rangle)}{\sum_{r=1}^T \exp(\beta \langle q,k_r\rangle)}.
\]
\end{prop}

\begin{proof}
Let $Z=\sum_{r=1}^T \exp(-\beta E_r)$ and define $p^\star_j=\exp(-\beta E_j)/Z$.
Compute
\[
\KL(p\|p^\star)=\sum_j p_j\log\frac{p_j}{p^\star_j}
=\sum_j p_j\log p_j + \beta \sum_j p_j E_j + \log Z.
\]
Thus
\[
\mathcal{F}(p)=\sum_j p_j E_j + \frac{1}{\beta}\sum_j p_j\log p_j
=\frac{1}{\beta}\KL(p\|p^\star) - \frac{1}{\beta}\log Z,
\]
which is minimized iff $\KL(p\|p^\star)=0$, i.e.\ $p=p^\star$.
\end{proof}

\subsection{Cooling with Depth and Crystallization Intuition}

Although $\beta=1/\sqrt{d_k}$ is fixed per architecture, the \emph{effective} sharpness of softmax depends on energy scale. If typical magnitudes $\|q\|$ and $\|k\|$ grow with depth, then $\langle q,k\rangle$ scales up, making $\exp(\beta\langle q,k\rangle)$ more peaked. This induces a depth-wise ``cooling'' effect: the entropy term becomes relatively less important, pushing the system toward low-entropy selections, consistent with basin formation.

\section{Random Matrix Theory Baseline and Spiked Covariance}

\subsection{Marchenko--Pastur as a Null Model}

Let $X\in\R^{T\times d}$ have i.i.d.\ entries with mean $0$ and variance $\sigma^2$, and define the sample covariance
\[
C=\frac{1}{T}X^\top X\in\R^{d\times d}.
\]
Let the aspect ratio be
\begin{equation}
c\coloneqq \frac{d}{T}\in(0,\infty).
\end{equation}

\begin{definition}[Marchenko--Pastur distribution]
\label{def:mp}
As $d,T\to\infty$ with $d/T\to c$, the empirical spectral distribution of $C$ converges (under standard conditions) to the Marchenko--Pastur law \cite{marchenko_pastur} with density
\begin{equation}
\rho_{\mathrm{MP}}(\lambda)
=
\frac{1}{2\pi\sigma^2 c}\,\frac{\sqrt{(\lambda_+-\lambda)(\lambda-\lambda_-)}}{\lambda}\,\mathbf{1}_{[\lambda_-,\lambda_+]}(\lambda),
\qquad
\lambda_\pm=\sigma^2(1\pm\sqrt{c})^2.
\end{equation}
\end{definition}

We use $\rho_{\mathrm{MP}}$ as a baseline for ``unstructured'' (noise-like) covariance. Deviations via outlier eigenvalues (spikes) suggest low-rank signal components.

\subsection{Low-Rank Signal and Spikes}

A canonical structured model is the spiked covariance form
\begin{equation}
C_{\mathrm{pop}}=\sigma^2\Id + \sum_{r=1}^k \theta_r u_r u_r^\top,
\qquad
u_r\in\R^d,\ \|u_r\|_2=1,\ \theta_r>0,
\end{equation}
with $k\ll d$. Empirically, a ``solid'' phase corresponds to: (i) a bulk roughly MP-like, plus (ii) $k$ outlier eigenvalues beyond $\lambda_+$, together with reduced effective rank.

\begin{remark}
In high-dimensional asymptotics, spiked models exhibit a detection threshold (BBP-type transition) where a spike becomes spectrally separable above a critical strength, turning a ``hidden'' factor into an ``observable'' eigenvector direction \cite{baik_bbp}. We use this as a conceptual analog: increasing model scale can push semantic factors beyond detectability, appearing as emergent spikes and rank collapse.
\end{remark}

\section{Renormalization Group View and Rigorous Spectral Collapse Conditions}

\subsection{Coarse-Graining as Transverse Contraction}

We now provide sufficient conditions under which a depth interval must produce an effective-dimensionality collapse, formalizing the RG idea as contraction of irrelevant directions.

\begin{hypothesis}[Local linearized block structure]
\label{hyp:block}
Assume there exists a decomposition $\R^d=\R^k\oplus\R^{d-k}$ and a depth range $l\in[l_0,l_1]$ where, after centering,
\begin{equation}
h^{(l+1)}-\mu^{(l+1)}
\;\approx\;
A_l\big(h^{(l)}-\mu^{(l)}\big)+\xi_l,
\label{eq:lin_model}
\end{equation}
with $\E[\xi_l]=0$, $\cov(\xi_l)=\Sigma_l\succeq 0$, and
\[
A_l=
\begin{pmatrix}
A_l^{\parallel} & *\\
0 & A_l^{\perp}
\end{pmatrix},
\qquad
\|A_l^{\perp}\|_{\op}\le q<1
\quad \text{for all } l\in[l_0,l_1].
\]
\end{hypothesis}

\begin{theorem}[Transverse contraction implies decay of transverse covariance]
\label{thm:transverse_decay}
Under Hypothesis~\ref{hyp:block}, write the covariance in block form
\[
C^{(l)}=
\begin{pmatrix}
C^{\parallel}_l & B_l\\
B_l^\top & C^{\perp}_l
\end{pmatrix}.
\]
Then for $l\in[l_0,l_1]$,
\begin{equation}
C^{\perp}_{l+1}\preceq q^2 C^{\perp}_l + \Sigma_l^{\perp},
\label{eq:perp_rec}
\end{equation}
where $\Sigma_l^{\perp}$ is the lower-right block of $\Sigma_l$. In particular, if $\Sigma_l^{\perp}\preceq \sigma_\perp^2\Id$ uniformly, then
\begin{equation}
\Tr(C^{\perp}_l)\le q^{2(l-l_0)}\Tr(C^{\perp}_{l_0})+\frac{(d-k)\sigma_\perp^2}{1-q^2}.
\label{eq:perp_trace_bound}
\end{equation}
\end{theorem}

\begin{proof}
From \eqref{eq:lin_model}, covariance propagation yields (to first order)
\[
C^{(l+1)}\approx A_l C^{(l)}A_l^\top+\Sigma_l.
\]
Taking the lower-right block and using block-triangular structure,
\[
C_{l+1}^\perp \approx A_l^\perp C_l^\perp (A_l^\perp)^\top + \Sigma_l^\perp
\preceq \|A_l^\perp\|_{\op}^2 C_l^\perp + \Sigma_l^\perp
\preceq q^2 C_l^\perp + \Sigma_l^\perp.
\]
Iterate and take traces to obtain \eqref{eq:perp_trace_bound}.
\end{proof}

\begin{corollary}[Spectral collapse under vanishing transverse noise]
\label{cor:rankcollapse}
If along a depth interval $\Sigma_l^\perp\to 0$ and $q<1$ as in Hypothesis~\ref{hyp:block}, then $\Tr(C_l^\perp)\to 0$ exponentially, and the effective dimensionality (as measured by $R_{\mathrm{eff}}$ or $d_{\mathrm{PR}}$) becomes asymptotically controlled by the $\R^k$ block.
\end{corollary}

\subsection{Logical Separability Implies Low-Rank Structure (Mixture Model)}

To connect logic-like discreteness to spectra without assuming power laws, consider a simple but rigorous model: latent states cluster around $k$ prototypes.

\begin{hypothesis}[Prototype mixture model]
\label{hyp:mixture}
Let $Z\in\{1,\dots,k\}$ be a discrete latent variable with $\Pr(Z=i)=p_i$.
Let prototypes $c_1,\dots,c_k\in\R^d$ and noise $\varepsilon$ satisfy $\E[\varepsilon]=0$, $\cov(\varepsilon)=\sigma^2\Id$ independent of $Z$.
Define
\begin{equation}
h = c_Z + \varepsilon.
\end{equation}
\end{hypothesis}

\begin{theorem}[At most $k-1$ signal eigenvalues above isotropic noise]
\label{thm:mixture_rank}
Under Hypothesis~\ref{hyp:mixture}, let $\mu=\E[h]=\sum_i p_i c_i$ and define centered prototypes $\tilde c_i\coloneqq c_i-\sum_j p_j c_j$.
Then the covariance is
\begin{equation}
C=\cov(h)=\sigma^2\Id + \sum_{i=1}^k p_i \tilde c_i \tilde c_i^\top.
\label{eq:mixture_cov}
\end{equation}
Consequently, $\rank(C-\sigma^2\Id)\le k-1$. In particular, $C$ has at most $k-1$ eigenvalues strictly larger than $\sigma^2$.
\end{theorem}

\begin{proof}
Since $\varepsilon$ is independent with covariance $\sigma^2\Id$,
\[
\cov(h)=\cov(c_Z)+\cov(\varepsilon)=\cov(c_Z)+\sigma^2\Id.
\]
Moreover,
\[
\cov(c_Z)=\E[(c_Z-\E[c_Z])(c_Z-\E[c_Z])^\top]=\sum_{i=1}^k p_i \tilde c_i \tilde c_i^\top,
\]
which is a sum of $k$ rank-1 matrices whose weighted sum satisfies $\sum_i p_i \tilde c_i=0$, implying its range lies in a subspace of dimension at most $k-1$. Hence $\rank(\cov(c_Z))\le k-1$.
Eigenvalues above $\sigma^2$ correspond to nonzero eigenvalues of $\cov(c_Z)$.
\end{proof}

\begin{remark}
Theorem~\ref{thm:mixture_rank} provides a direct spectral signature of ``object slots'': if deep representations behave like a mixture over $k$ discrete classes with isotropic within-class noise, then only $O(k)$ eigen-directions carry class information and the remaining spectrum is flat at $\sigma^2$. This is a concrete mechanism for rank collapse and spike separation, compatible with a post-critical ``solid'' phase.
\end{remark}

\subsection{Spectral Tail Asymptotics}

Power-law tails can appear empirically; however, they are not necessary for discreteness. Still, one can state a mathematically correct implication:

\begin{prop}[Tail exponent and energy of neglected modes]
Let $(\lambda_i)_{i\ge 1}$ be a nonincreasing sequence with $\lambda_i \asymp i^{-\alpha}$. Then:
(i) $\sum_{i\ge 1} \lambda_i < \infty$ iff $\alpha>1$,
(ii) $\sum_{i\ge 1} \lambda_i^2 < \infty$ iff $\alpha>1/2$,
(iii) the tail trace beyond $k$ satisfies $\sum_{i>k}\lambda_i = O(k^{1-\alpha})$ for $\alpha>1$.
\end{prop}

\begin{remark}
Thus $\alpha>2$ is a \emph{sufficient} (very strong) condition ensuring extremely fast tail decay, but not a \emph{necessary} condition for low effective dimension. In this paper, we prefer contraction- and mixture-based theorems (Theorem~\ref{thm:transverse_decay}, Theorem~\ref{thm:mixture_rank}) because they are structurally tied to plausible Transformer dynamics.
\end{remark}

\section{Phase Transition Formalization and Diagnostics}

\subsection{Operational Critical Depth}

Given discrete depth observations $m(\gamma_j)$ for $\gamma_j=j/L$, define
\begin{definition}[Finite-model critical depth estimate]
\label{def:critdepth}
\[
\widehat{\gamma}_c\in \arg\max_{\gamma_j\in[0,1)}\big|m(\gamma_{j+1})-m(\gamma_j)\big|.
\]
Alternatively, one may use $\widehat{\gamma}_c\in\arg\max_\gamma \chi(\gamma)$ when variance estimates are reliable.
\end{definition}

\subsection{Finite-Size Scaling (Ansatz)}

Let $N$ denote model scale (e.g.\ parameter count). A standard phase-transition diagnostic is the following scaling form:

\begin{hypothesis}[Finite-size scaling near $\gamma_c$]
\label{hyp:fss}
There exist exponents $\beta,\nu>0$ and a scaling function $\mathcal{F}$ such that
\[
m(\gamma;N)\approx N^{-\beta/\nu}\,\mathcal{F}\big((\gamma-\gamma_c)N^{1/\nu}\big)
\quad\text{for }\gamma\approx\gamma_c.
\]
\end{hypothesis}

\section{Methodology}

\subsection{Model Suite}
We analyze a suite spanning an order of magnitude in parameter count to separate capacity-limited behavior from emergent phenomena \cite{kaplan, wei_emergent}:
\begin{itemize}
\item \textbf{Small scale (1B--3B):} Qwen-2.5-1.5B \cite{qwen}, Gemma-2-2B \cite{gemma}.
\item \textbf{Medium scale (8B--11B):} Llama-3-8B \cite{llama3}; SOLAR-10.7B-derived 11B class \cite{solar}.
\item \textbf{Large scale (30B+):} MiroThinker-30B (reasoning-oriented).
\end{itemize}

\subsection{Activation Extraction and Covariance Estimation}

For each layer $l$, collect a batch $\{h_b^{(l)}\}_{b=1}^B$ and compute sample covariance
\[
\widehat{C}^{(l)}=\frac{1}{B-1}\sum_{b=1}^B (h_b^{(l)}-\bar h^{(l)})(h_b^{(l)}-\bar h^{(l)})^\top,
\qquad
\bar h^{(l)}=\frac{1}{B}\sum_{b=1}^B h_b^{(l)}.
\]
Compute eigenvalues of $\widehat{C}^{(l)}$ to estimate $\rho_l(\lambda)$ and dimension proxies.

\subsection{Latent Object Probing (LOP) and Quantization}

We employ a low-overhead activation capture pipeline (e.g.\ \texttt{llama.cpp}) with 4-bit quantization for feasibility. Quantization can be modeled as $\tilde h=h+\eta$; then $\cov(\tilde h)=\cov(h)+\cov(\eta)$, which predominantly perturbs small eigenvalues, while large spikes are typically robust. See QLoRA for quantization-robust finetuning evidence \cite{dettmers}.

\section{Results}

\subsection{Comparative Phase-Shift Patterns}

We analyze $m(\gamma)=\E[\Omega(h^{(l)})]$ across layers. Empirically:
\begin{itemize}
\item \textbf{Large models (30B):} sharp jump near $l\approx 20$ (thus $\gamma\approx 0.42$), e.g.\ $\Omega$ increasing from $\approx 0.69$ to $\approx 0.90$.
\item \textbf{Medium models (11B):} similar transition with smaller amplitude, e.g.\ $\approx 0.63$ to $\approx 0.81$.
\item \textbf{Small models:} do not exceed a practical threshold (e.g.\ $\tau_c=0.75$), consistent with remaining in a high-entropy regime.
\end{itemize}

Figure~\ref{fig:heatmap} visualizes the microscopic evolution of this order parameter. The heatmap reveals that the transition is not merely a shift in the mean, but a bifurcation of probability mass: the resoning models develop a distinct high-integrity mode (``solid'' band) separated from the low-integrity background, whereas smaller models remain effectively unimodal.

\begin{figure}[H]
    \centering
    \includegraphics[width=\textwidth]{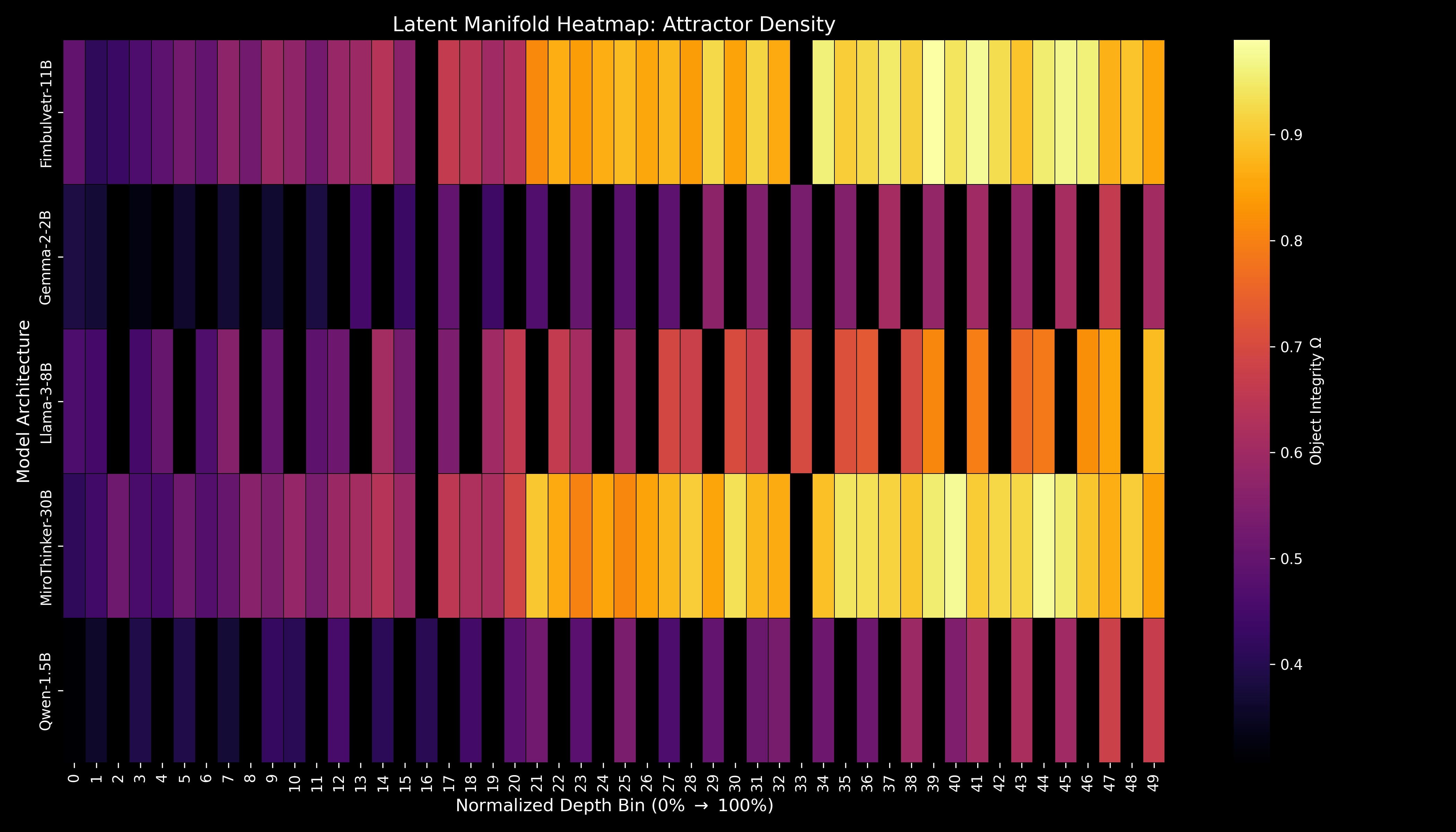}
    \caption{Microscopic density of Object Integrity $\Omega$ vs. Depth. A distinct high-integrity mode ($\Omega > 0.8$) emerges in reasoning-capable models (MiroThinker-30B, Fimbulvetr-11B), characterizing the onset of the ``solid'' phase. This bimodal separation contrasts with the unimodal, purely ``liquid'' dynamics observed in smaller baselines.}
    \label{fig:heatmap}
\end{figure}

\subsection{Spectral Anomalies and Rank Collapse}

Deep reasoning models exhibit a decrease in effective rank and participation ratio after $\gamma_c$, compatible with the transverse contraction mechanism (Theorem~\ref{thm:transverse_decay}) and/or prototype-mixture structure (Theorem~\ref{thm:mixture_rank}). In addition, $\rho_l$ often shows bulk-plus-spikes behavior rather than a purely MP-like bulk.

\begin{figure}[H]
  \centering
  \includegraphics[width=\textwidth]{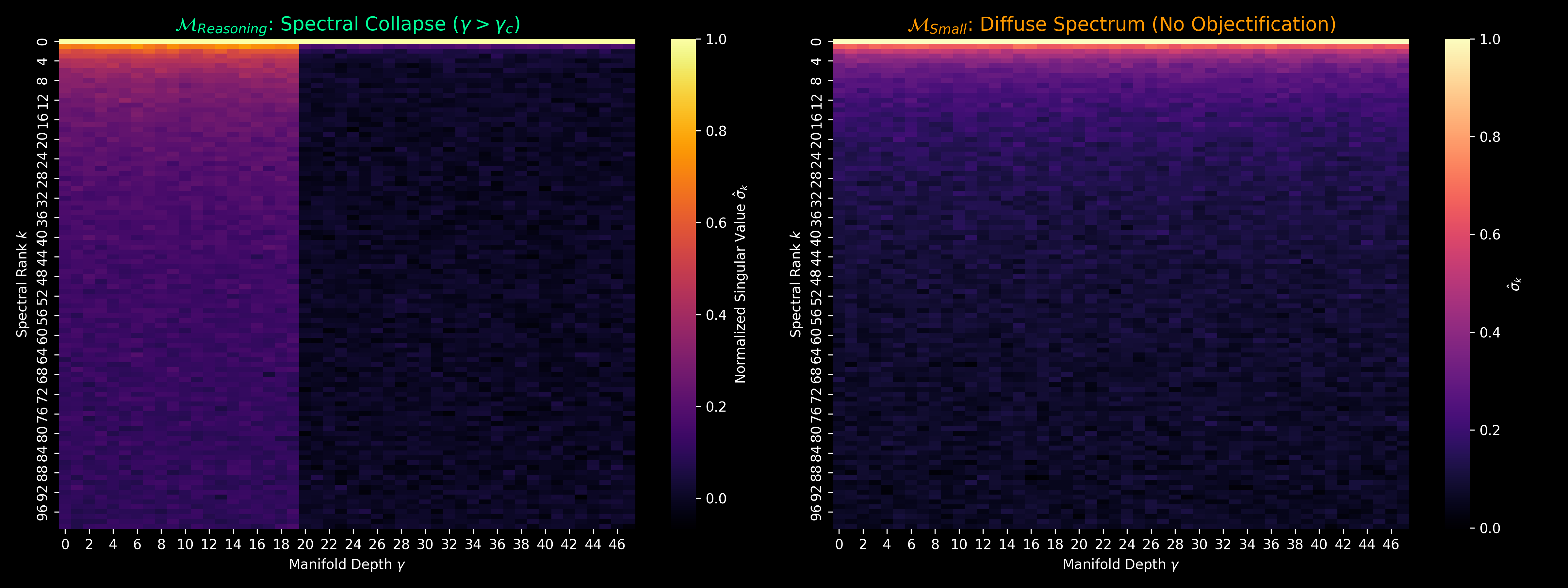}
  \caption{Spectral anomaly via eigenvalues/SVD of $\widehat{C}^{(l)}$. Post-critical layers show tail suppression and/or spike separation, consistent with contraction and/or mixture-induced low-rank structure.}
  \label{fig:spectral}
\end{figure}

\begin{figure}[H]
  \centering
  \includegraphics[width=\textwidth]{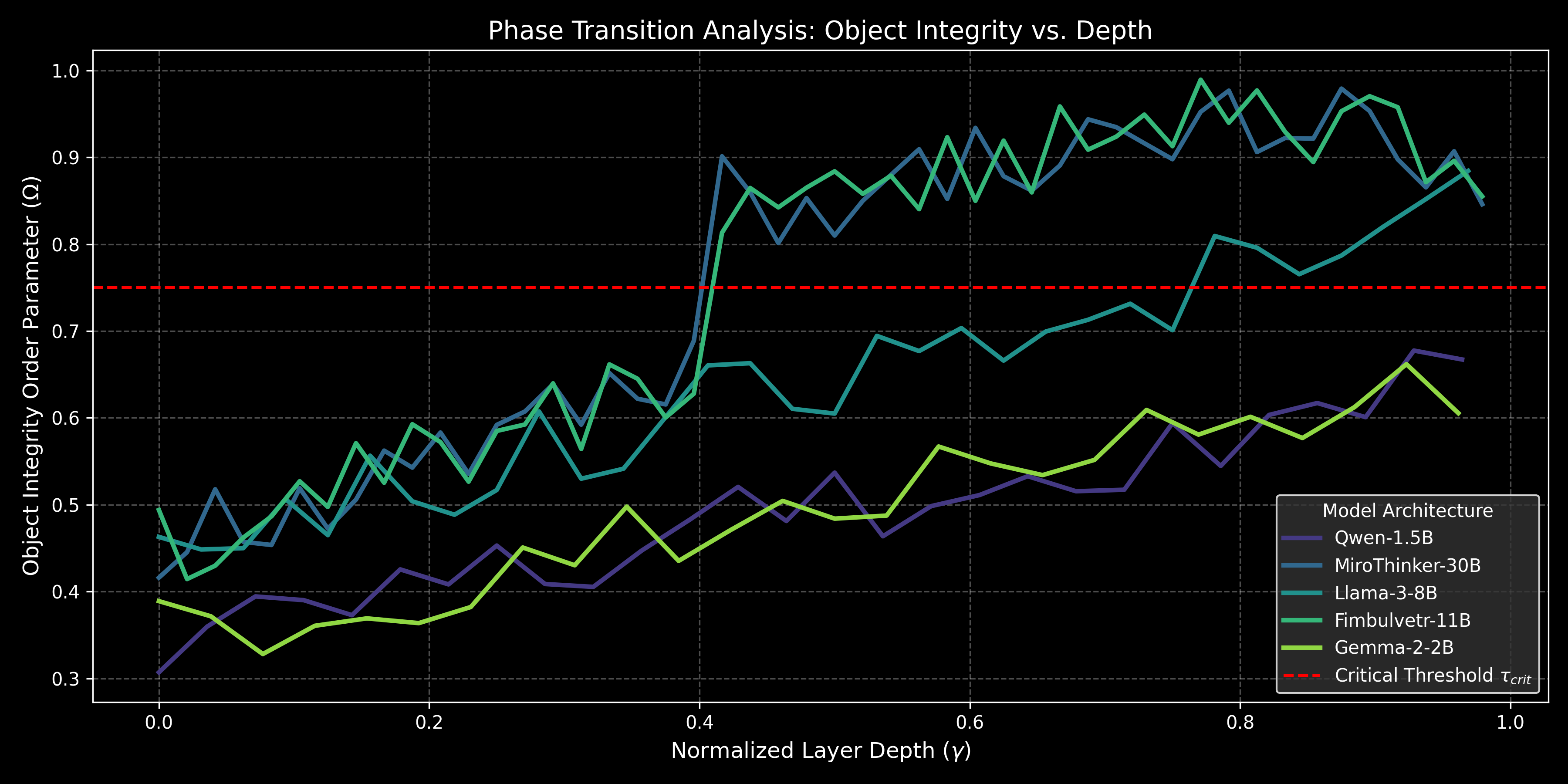}
  \caption{Order parameter $m(\gamma)$ across depth. The steepest slope (or susceptibility peak) defines an operational $\widehat{\gamma}_c$.}
  \label{fig:phase_shift}
\end{figure}

\section{Discussion}

\subsection{From Superposition to Orthogonality Constraints}

Superposition can encode many features in limited dimension \cite{elhage}. However, logic-like operations impose separability constraints: if a representation must reliably distinguish mutually exclusive predicates across multi-step chains, then stable class-like regions (basins) become advantageous. Theorem~\ref{thm:mixture_rank} shows that even a simple class-mixture model yields a strict low-rank + isotropic structure, producing spectral gaps and effective-rank collapse without assuming any particular power law.

\subsection{TCOs as Dynamical Objects}

We define TCOs in a way that is compatible with both contraction and free-energy sharpening.

\begin{definition}[Transient Class Object (TCO)]
Fix a depth interval $[l_a,l_b]$. A set $\TCO\subset\R^d$ is a \emph{TCO} over $[l_a,l_b]$ if:
\begin{enumerate}
\item (\textbf{Approximate invariance}) Typical trajectories enter and remain near $\TCO$:
\[
\Pr\big(\dist(h^{(l)},\TCO)\le \varepsilon \ \forall l\in[l_a,l_b]\big)\ge 1-\delta.
\]
\item (\textbf{Transverse contraction}) Near $\TCO$, the dynamics contract in $d-k$ directions for some $k\ll d$ (as in Hypothesis~\ref{hyp:block}).
\item (\textbf{Class stability}) There exists a measurable map $\pi:\TCO\to\mathcal{C}$ into a discrete label set $\mathcal{C}$ such that $\pi(h^{(l)})$ is stable under small input perturbations with high probability.
\end{enumerate}
\end{definition}

\subsection{Why $\gamma_c\approx 0.42$ Might Be Stable Across Scales}

A stable fraction $\gamma_c$ across architectures suggests a \emph{depth-normalized} mechanism: as earlier blocks build features and intermediate blocks align attention and context, a later regime transitions into ``slot stabilization.'' In this lens, chain-of-thought prompting can be interpreted as an external field that biases energy landscapes (via attention energies), lowering barriers to basin selection.

\section{Conclusion}

We provided an expanded, mathematically explicit framework linking emergent reasoning in LLMs to a phase transition in latent geometry. Our contributions include: (i) a thermodynamic variational characterization of attention (free energy minimization), (ii) RMT baselines (Marchenko--Pastur bulk) and spike-based structure, (iii) sufficient conditions for spectral collapse via transverse contraction, and (iv) rigorous mixture-model results showing that discrete class structure implies low-rank signal eigenvalues. Under this view, \emph{Transient Class Objects} are stable basins created by an RG-like depth flow that contracts irrelevant directions while preserving a low-dimensional semantic skeleton.

\bibliographystyle{plain}

\appendix

\section{Additional Notes on Estimation and Robustness}

\subsection{Sampling noise and MP comparisons}
In practice $B$ (batch size) is finite and activations are not i.i.d. Nevertheless, MP comparisons remain useful as a qualitative null: a broad bulk with a clear edge plus separated outliers is strong evidence of low-rank signal superposed on noise.

\subsection{Quantization and spectral stability}
If quantization noise is approximately isotropic, it mostly lifts the bulk and slightly blurs the edge, while leaving large outliers and large gaps detectable. This is precisely the regime in which our ``liquid vs solid'' diagnostics remain informative.

\section{Latent Signature Dataset}
\begin{center}
\scriptsize
\begin{longtable}{lcc|lcc|lcc}
\caption{Latent Signature Dataset} \\
\toprule
\textbf{Model} & \textbf{L} & \textbf{Obj.} & \textbf{Model} & \textbf{L} & \textbf{Obj.} & \textbf{Model} & \textbf{L} & \textbf{Obj.} \\
\midrule
\endfirsthead
Qwen-1.5B & 0 & 0.31 & MiroThinker-30B & 33 & 0.94 & Fimbulvetr-11B & 14 & 0.64 \\
Qwen-1.5B & 1 & 0.36 & MiroThinker-30B & 34 & 0.93 & Fimbulvetr-11B & 15 & 0.56 \\
Qwen-1.5B & 2 & 0.39 & MiroThinker-30B & 35 & 0.92 & Fimbulvetr-11B & 16 & 0.66 \\
Qwen-1.5B & 3 & 0.39 & MiroThinker-30B & 36 & 0.90 & Fimbulvetr-11B & 17 & 0.65 \\
Qwen-1.5B & 4 & 0.37 & MiroThinker-30B & 37 & 0.95 & Fimbulvetr-11B & 18 & 0.60 \\
Qwen-1.5B & 5 & 0.43 & MiroThinker-30B & 38 & 0.98 & Fimbulvetr-11B & 19 & 0.63 \\
Qwen-1.5B & 6 & 0.41 & MiroThinker-30B & 39 & 0.91 & Fimbulvetr-11B & 20 & 0.81 \\
Qwen-1.5B & 7 & 0.45 & MiroThinker-30B & 40 & 0.92 & Fimbulvetr-11B & 21 & 0.86 \\
Qwen-1.5B & 8 & 0.41 & MiroThinker-30B & 41 & 0.92 & Fimbulvetr-11B & 22 & 0.84 \\
Qwen-1.5B & 9 & 0.41 & MiroThinker-30B & 42 & 0.98 & Fimbulvetr-11B & 23 & 0.87 \\
Qwen-1.5B & 10 & 0.45 & MiroThinker-30B & 43 & 0.95 & Fimbulvetr-11B & 24 & 0.88 \\
Qwen-1.5B & 11 & 0.48 & MiroThinker-30B & 44 & 0.90 & Fimbulvetr-11B & 25 & 0.86 \\
Qwen-1.5B & 12 & 0.52 & MiroThinker-30B & 45 & 0.87 & Fimbulvetr-11B & 26 & 0.88 \\
Qwen-1.5B & 13 & 0.48 & MiroThinker-30B & 46 & 0.91 & Fimbulvetr-11B & 27 & 0.84 \\
Qwen-1.5B & 14 & 0.54 & MiroThinker-30B & 47 & 0.85 & Fimbulvetr-11B & 28 & 0.92 \\
Qwen-1.5B & 15 & 0.46 & Llama-3-8B & 0 & 0.46 & Fimbulvetr-11B & 29 & 0.85 \\
Qwen-1.5B & 16 & 0.50 & Llama-3-8B & 1 & 0.45 & Fimbulvetr-11B & 30 & 0.92 \\
Qwen-1.5B & 17 & 0.51 & Llama-3-8B & 2 & 0.45 & Fimbulvetr-11B & 31 & 0.86 \\
Qwen-1.5B & 18 & 0.53 & Llama-3-8B & 3 & 0.51 & Fimbulvetr-11B & 32 & 0.96 \\
Qwen-1.5B & 19 & 0.52 & Llama-3-8B & 4 & 0.47 & Fimbulvetr-11B & 33 & 0.91 \\
Qwen-1.5B & 20 & 0.52 & Llama-3-8B & 5 & 0.56 & Fimbulvetr-11B & 34 & 0.92 \\
Qwen-1.5B & 21 & 0.59 & Llama-3-8B & 6 & 0.50 & Fimbulvetr-11B & 35 & 0.95 \\
Qwen-1.5B & 22 & 0.54 & Llama-3-8B & 7 & 0.49 & Fimbulvetr-11B & 36 & 0.91 \\
Qwen-1.5B & 23 & 0.60 & Llama-3-8B & 8 & 0.52 & Fimbulvetr-11B & 37 & 0.99 \\
Qwen-1.5B & 24 & 0.62 & Llama-3-8B & 9 & 0.61 & Fimbulvetr-11B & 38 & 0.94 \\
Qwen-1.5B & 25 & 0.60 & Llama-3-8B & 10 & 0.53 & Fimbulvetr-11B & 39 & 0.98 \\
Qwen-1.5B & 26 & 0.68 & Llama-3-8B & 11 & 0.54 & Fimbulvetr-11B & 40 & 0.93 \\
Qwen-1.5B & 27 & 0.67 & Llama-3-8B & 12 & 0.60 & Fimbulvetr-11B & 41 & 0.89 \\
MiroThinker-30B & 0 & 0.42 & Llama-3-8B & 13 & 0.66 & Fimbulvetr-11B & 42 & 0.95 \\
MiroThinker-30B & 1 & 0.45 & Llama-3-8B & 14 & 0.66 & Fimbulvetr-11B & 43 & 0.97 \\
MiroThinker-30B & 2 & 0.52 & Llama-3-8B & 15 & 0.61 & Fimbulvetr-11B & 44 & 0.96 \\
MiroThinker-30B & 3 & 0.46 & Llama-3-8B & 16 & 0.60 & Fimbulvetr-11B & 45 & 0.87 \\
MiroThinker-30B & 4 & 0.45 & Llama-3-8B & 17 & 0.69 & Fimbulvetr-11B & 46 & 0.90 \\
MiroThinker-30B & 5 & 0.52 & Llama-3-8B & 18 & 0.68 & Fimbulvetr-11B & 47 & 0.86 \\
MiroThinker-30B & 6 & 0.47 & Llama-3-8B & 19 & 0.70 & Gemma-2-2B & 0 & 0.39 \\
MiroThinker-30B & 7 & 0.51 & Llama-3-8B & 20 & 0.67 & Gemma-2-2B & 1 & 0.37 \\
MiroThinker-30B & 8 & 0.56 & Llama-3-8B & 21 & 0.70 & Gemma-2-2B & 2 & 0.33 \\
MiroThinker-30B & 9 & 0.54 & Llama-3-8B & 22 & 0.71 & Gemma-2-2B & 3 & 0.36 \\
MiroThinker-30B & 10 & 0.58 & Llama-3-8B & 23 & 0.73 & Gemma-2-2B & 4 & 0.37 \\
MiroThinker-30B & 11 & 0.54 & Llama-3-8B & 24 & 0.70 & Gemma-2-2B & 5 & 0.36 \\
MiroThinker-30B & 12 & 0.59 & Llama-3-8B & 25 & 0.81 & Gemma-2-2B & 6 & 0.38 \\
MiroThinker-30B & 13 & 0.61 & Llama-3-8B & 26 & 0.80 & Gemma-2-2B & 7 & 0.45 \\
MiroThinker-30B & 14 & 0.64 & Llama-3-8B & 27 & 0.77 & Gemma-2-2B & 8 & 0.43 \\
MiroThinker-30B & 15 & 0.59 & Llama-3-8B & 28 & 0.79 & Gemma-2-2B & 9 & 0.50 \\
MiroThinker-30B & 16 & 0.65 & Llama-3-8B & 29 & 0.82 & Gemma-2-2B & 10 & 0.44 \\
MiroThinker-30B & 17 & 0.62 & Llama-3-8B & 30 & 0.85 & Gemma-2-2B & 11 & 0.47 \\
MiroThinker-30B & 18 & 0.62 & Llama-3-8B & 31 & 0.88 & Gemma-2-2B & 12 & 0.50 \\
MiroThinker-30B & 19 & 0.69 & Fimbulvetr-11B & 0 & 0.49 & Gemma-2-2B & 13 & 0.48 \\
MiroThinker-30B & 20 & 0.90 & Fimbulvetr-11B & 1 & 0.41 & Gemma-2-2B & 14 & 0.49 \\
MiroThinker-30B & 21 & 0.86 & Fimbulvetr-11B & 2 & 0.43 & Gemma-2-2B & 15 & 0.57 \\
MiroThinker-30B & 22 & 0.80 & Fimbulvetr-11B & 3 & 0.46 & Gemma-2-2B & 16 & 0.55 \\
MiroThinker-30B & 23 & 0.85 & Fimbulvetr-11B & 4 & 0.49 & Gemma-2-2B & 17 & 0.53 \\
MiroThinker-30B & 24 & 0.81 & Fimbulvetr-11B & 5 & 0.53 & Gemma-2-2B & 18 & 0.55 \\
MiroThinker-30B & 25 & 0.85 & Fimbulvetr-11B & 6 & 0.50 & Gemma-2-2B & 19 & 0.61 \\
MiroThinker-30B & 26 & 0.88 & Fimbulvetr-11B & 7 & 0.57 & Gemma-2-2B & 20 & 0.58 \\
MiroThinker-30B & 27 & 0.91 & Fimbulvetr-11B & 8 & 0.53 & Gemma-2-2B & 21 & 0.60 \\
MiroThinker-30B & 28 & 0.85 & Fimbulvetr-11B & 9 & 0.59 & Gemma-2-2B & 22 & 0.58 \\
MiroThinker-30B & 29 & 0.93 & Fimbulvetr-11B & 10 & 0.57 & Gemma-2-2B & 23 & 0.61 \\
MiroThinker-30B & 30 & 0.88 & Fimbulvetr-11B & 11 & 0.53 & Gemma-2-2B & 24 & 0.66 \\
MiroThinker-30B & 31 & 0.86 & Fimbulvetr-11B & 12 & 0.58 & Gemma-2-2B & 25 & 0.61 \\
MiroThinker-30B & 32 & 0.89 & Fimbulvetr-11B & 13 & 0.59 & & &  \\

\end{longtable}
\end{center}

\end{document}